\let\accentvec\vec
\documentclass[twocolumn]{svjour3}         % twocolumn

\let\vec\accentvec
\smartqed  % flush right qed marks, e.g. at end of proof
\usepackage{mathptmx}      % use Times fonts if available on your TeX system
\usepackage{latexsym}
\usepackage{amsmath,amssymb} % define this before the line numbering.
\usepackage{epsfig}
\usepackage{graphicx}
\usepackage[noend]{algorithmic}
\usepackage{color}
\usepackage{algorithm}
\usepackage{natbib}

\newtheorem{thm}{Theorem}

\newtheorem{lem}[thm]{Lemma}

\newtheorem{defn}[thm]{Definition}
\newtheorem{rem}[thm]{Remark}
\newcommand{\norm}[1]{\left\Vert#1\right\Vert}

\newcommand{\set}[1]{\left\{#1\right\}}

\newcommand{\E}{\mathbb E}

\journalname{Arxiv.org}
\begin{document}

\title{Large Margin Boltzmann Machines and Large Margin Sigmoid Belief Networks%\thanks{Grants or other notes
}
%\subtitle{Fast Structured Prediction}

%\titlerunning{Fast Structured Prediction}        % if too long for running head

\author{Xu Miao         \and
        Rajesh P.N. Rao %etc.
}


\institute{Xu Miao \at
              Computer Science and Engineering,\\
              University of Washington, Seattle WA,98125 USA\\
              Tel.: +1-206-685-2035\\
              Fax: +1-206-543-2969\\
              \email{xm@cs.washington.edu}           %  \\
           \and
           Rajesh P.N. Rao \at
           Computer Science and Engineering, \\
           University of Washington, Seattle WA,98125 USA
}

\date{Received: date / Accepted: date}

\maketitle

\begin{abstract}
Current statistical
models for structured prediction make simplifying assumptions about
the underlying output graph structure, such as assuming a low-order
Markov chain, because exact inference becomes intractable as the
tree-width of the underlying graph increases. Approximate inference
algorithms, on the other hand, force one to trade off representational
power with computational efficiency. In this paper, we propose two new types of probabilistic
graphical models, large margin Boltzmann machines (LMBMs) and large
margin
sigmoid belief networks (LMSBNs), for structured prediction.
LMSBNs in particular  allow a very fast inference algorithm for
arbitrary graph structures that runs in polynomial time with a high
probability. This probability is data-distribution dependent and is
maximized in learning. The new approach overcomes the
representation-efficiency trade-off in previous models and allows fast structured prediction with complicated
graph structures. We
present results from applying a fully connected model to multi-label
scene classification and demonstrate that the proposed approach can yield significant performance
gains over current state-of-the-art methods.
\keywords{Structured Prediction \and Probabilistic Graphical Models \and Exact and Approximate Inference}
\end{abstract}

\section{Introduction}\label{sec:intro}
Structured prediction is an important machine learning problem that
occurs in many different fields, e.g., natural language processing,
protein structure prediction and semantic image annotation. The goal is to learn a function
that maps an input vector $\bf X$ to an output $\bf Y$, where
$\bf Y$ is a vector representing all the labels whose components take
on the value $+1$ or $-1$ (presence or absence of the corresponding
label). The traditional approach to such multi-label classification
problems is to train a set of binary classifiers
independently. Structured prediction on the other hand also considers
the relationships among the output variables $\bf Y$. For example, in the image annotation problem, an entire image or parts of an image are annotated with labels representing an object, a scene or an event involving multiple objects
 \citep{Carneiro07}. These labels are usually dependent on each other, e.g., buildings and beaches occur under the sky, a truck is a type
of automotive, and sunsets are more likely to co-occur with beaches,
sky, and trees (Figure~\ref{fig:annotation}). Such relations capture
the semantics among the labels and play an important role in human
cognition. A major advantage of structured prediction is that the structured
representation of the output can be much more compact than an
unstructured classifier, resulting in smaller sample complexity and
greater generalization \citep{Bengio07greedylayer-wise}.

\begin{figure*}
\centering
	\includegraphics[width=0.7\linewidth]{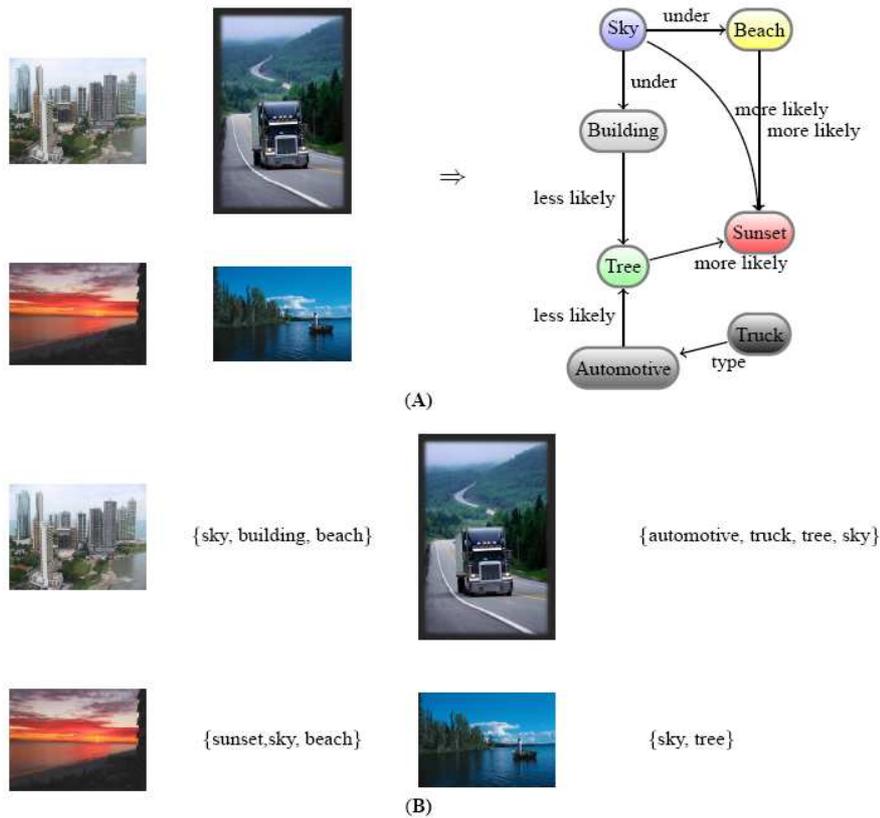}
  \caption{Semantic Image Annotation. (A). During training, we train a function mapping an image to the label semantics. (B). During prediction, we annotate each image a set of labels.}
  \label{fig:annotation}
\end{figure*}

Extending traditional classification techniques to structured
prediction is difficult because of the potentially complicated
inter-dependencies that may exist among the output variables. If the
problem is modeled as a \emph{probabilistic graphical model}, it is
well-known that exact inference over a general graph is
NP-hard. Therefore, practical approaches make simplifying assumptions
about the dependencies among the output variables in order to
simplify the graph structure and maintain tractability. Examples
include \emph{maximum entropy Markov models} 
(MEMMs) \citep{Mccallum00}, \emph{conditional random fields} 
(CRFs) \citep{Lafferty01,Quattoni04conditionalrandom}, \emph{max-margin Markov networks} (M3Ns) \citep{Taskar03} and \emph{structured support vector machines} (SSVMs) \citep{Tsochantaridis04}. These approaches typically restrict
the tree-width\footnote{In this paper, the tree-width of a
\emph{directed acyclic graph} refers to the tree-width of the
corresponding undirected graph obtained through moralization.} of the
graph so that the \emph{Viterbi} algorithm or the \emph{junction tree}
algorithm can still be efficient. 

On the other hand, there has been much research on fast approximate
inference for complicated graphs based on, e.g., \emph{Markov chain
Monte Carlo} (MCMC), \emph{variational inference}, or combinations of
these methods. In general, MCMC is slow, particularly for graphs with
strongly coupled variables. Good heuristics have been developed to
speed up MCMC, but they are highly dependent on graph structure and
associated parameters \citep{Doucet00}. \emph{Variational inference} is
another popular approach where a complicated distribution over $\bf Y$
is approximated with a simpler distribution so as to trade accuracy
for speed. For example, if the variables are assumed to be
independent, one obtains the \emph{mean field} algorithm. A Bethe
energy formulation yields the \emph{loopy belief propagation} (LBP)
algorithm \citep{Yedidia05}. If a combination of trees is considered,
one obtains the \emph{tree-reweighted sum-product}
algorithm \citep{Wainwright05b}. One can also relax the higher-order
marginal constraints to obtain a \emph{linear programming}
algorithm \citep{Wainwright05a}. The lesser the dependency constraints,
the less accurate these inference algorithms become, and the faster
their speed. However, the sacrificed accuracy in inference could be
detrimental to learning. For example, \emph{mean field} can produce
highly biased estimates, and \emph{loopy belief propagation} might
even cause the learning algorithm to diverge\\~\citep{Kulesza2007}.

Long-range dependencies and complicated
graphs are necessary to accurately and precisely represent semantic knowledge. Unfortunately, the approaches discussed above all operate under the assumption that one cannot avoid the trade-off between the
representational power and computational efficiency.

In this paper, we propose \emph{large margin sigmoid belief
 networks} (LMSBNs) and \emph{large margin Boltzmann machines} (LMBMs), two new models for structured prediction. We provide a theoretical analysis tool to derive the generalization bounds for both of them. Most importantly, LMSBNs  allow fast inference for arbitrarily complicated graph
 structures. Inference is based on a \emph{branch-and-bound} (BB)
 technique that does not depend on the dependency structure of the
 graph and exhibits the interesting property that the better the fit
 of the model to the data, the faster the inference procedure. 

Section~\ref{sec:lmsbnbm} describes both LMSBNs and LMBMs. 
We present learning algorithms for both and the fast BB inference algorithm for LMSBNs. LMBMs, being undirected, rely on traditional inference algorithms. 

Section~\ref{sec:experiments} applies both LMSBNs and LMBMs to the semantic
image annotation problem using a fully-connected graph structure.  We
empirically study the performance of the BB inference algorithm and
illustrate its efficiency and effectiveness. We present results from
experiments on a benchmark dataset which demonstrate that LMSBNs
outperform current state-of-the-art methods for image annotation based
on kernels and threshold-tuning.

\section{Large Margin Sigmoid Belief Networks and Large Margin Boltzmann Machines}\label{sec:lmsbnbm}
The \emph{sigmoid belief network} (SBN) \citep{Neal92} and \emph{Boltzmann machine} (BM) \citep{Hinton83} are a special
type of \emph{Bayesian network} and a special type of \emph{Markov random field} respectively, and are defined as follows:
\begin{defn}\label{defn:bm}
A \emph{Boltzmann machine} is an undirected graph $G=({\bf V},E)$, where $\bf V$ is the set of random variables with size $K = |\bf V|$, $E$ is the set of undirected edges. The joint likelihood is defined as:
\begin{eqnarray}\label{eq:bml}
\Pr({\bf V}|{\bf w}) &=& e^{\frac{1}{2}\sum_i z_i}/\sum_{\bf V}e^{\frac{1}{2}\sum_iz_i}\\
z_i&=& \sum_{j:(V_i,V_j)\in E}w_{ij}v_iv_j+ w_iv_i\nonumber
\end{eqnarray}
where $Z$ is the normalization constant.
\end{defn}
\begin{defn}\label{defn:sbn}
A \emph{sigmoid belief network} is a \emph{directed
acyclic graph} $G=({\bf V},E)$, where $\bf V$ is the set of random variables with size $K = |\bf V|$, $E$ is the set of directed edges. $(V_j,V_i)$ represents an edge from $V_j$ to $V_i$. For
each node $V_i$, its parents are in the set $pa(V_i)=\{V_j|(V_j,V_i)\in E\}$. The joint likelihood is defined as:
\begin{eqnarray}\label{eq:dcl}
\Pr({\bf V}|{\bf w})&=&\prod_{i=1}^K\Pr(V_i|pa(V_i),{\bf w})\\
\Pr(V_i|pa(V_i),{\bf w})&=&\frac{1}{1+e^{-z_i}}\nonumber\\
z_i &=& \sum_{j:V_j \in pa(V_i)}w_{ij}v_iv_j+w_iv_i\nonumber
\end{eqnarray}
\end{defn}

In BMs, the edges are undirected, so the feature $v_iv_j$ appears in both $z_i$ and $z_j$. In SBNs, the edges are directed, so the feature $v_iv_j$ appears in either $z_i$ or $z_j$, but not both. One can generalize the function $z_i$ to utilize high order features over a set of variables. In \emph{probabilistic graphical models}, this set is
referred to as a \emph{clique}. In SBNs or BMs, the features are defined as a product of all variables in the clique. For example, $C_1=\{V_1,V_2,V_3\}$ is a 3rd order clique, $f_1=v_1v_2v_3$. The edges are 2nd order cliques, e.g., $C_2=\{V_1,V_2\}$, $f_2=v_1v_2$. The first order cliques are the variable themselves, e.g., $C_3=\{V_1\}$, $f_3=v_1$. When the variables take values
$\{-1,1\}$, the feature function is also known as the \emph{parity}
function or the XOR function. Therefore, a SBN or BM softly encodes a
Boolean function via an AND-of-XOR expansion\footnote{This is
different from the \emph{ring-sum expansion} which is an XOR-of-AND
expansion.}, which provides a flexible way to encode human expert knowledge
into the model. Without ambiguity, we simplify the representation of $z_i$ to be $z_i=\sum_jw_{ij}f_j$, where the summation is taken over all cliques that include variable $V_i$. For SBNs, We require that all the variables in each clique $C_j$ other than $V_i$ must be parents of $V_i$. This requirement insures that the underlying graph is acyclic, and each $C_j$ is used in one $z_i$. 

In the structured prediction setting, the problem involves an input
vector $\bf X$, and the joint probability over all $\bf Y$ is conditioned
on $\bf X$, i.e., $\Pr({\bf Y}|\bf X, w)$. Note that $z_i$ is defined for each $Y_i$ although the cliques include both $\bf X$ and $\bf Y$.

When there is only one output variable, i.e. $K=1$, the conditional likelihoods of both SBNs and BMs become the same, i.e., $\Pr(Y=y|{\bf
x};{\bf w})=\frac{1}{1+e^{-z}}$, where $z = y\sum_jw_j\phi_j({\bf
x})$. The features are $f_j({\bf x},y)=y\phi_j({\bf x})$. This is the well known \emph{logistic regression} (LR) with a loss function $L(y,{\bf
x,w})=\log(1+e^{-z})$. In fact, a SBN can be considered as a product
of LRs according to a topological order over the graph. The overall
loss function is then $L({\bf y,x,w}) = \sum_i\log(1+e^{-z_i})$. A BM needs normalization over all $\bf Y$, the loss function usually can not be factorized locally that puts some challenge on learning.

To facilitate the derivation of a fast inference algorithm for LMSBNs and a fast learning algorithm for LMBMs, we use a hinge loss, $[1-z]_+ = \max(0,1-z)$ to approximate the log-loss $\log(1+e^{-z})$. We call the resulting SBN a \emph{large margin
sigmoid belief network} (LMSBN) and the resulting BM a \emph{large margin Boltzmann machine} (LMBM). The approximations are presented in Remark~\ref{rem:approx}. The approximation of LMBM is similar to \emph{pseudo likelihood} approximation of a \emph{Markov random field}. The only difference is the extra regularization. In the latter section, we will show that this regularizer is crucial for LMBMs to generalize well. Note that for LMSBNs, each feature $f_j$ only appears in one $z_i$, but for LMBMs, each feature $f_j$ appears in all $z_i$ where $Y_i\in C_j$. 

\begin{rem}
\begin{eqnarray}
	L_{SBN}({\bf y,x,w}) &\le& L_{LMSBN}({\bf y,x,w}) + Kb\nonumber\\
	L_{BM}({\bf y,x,w}) &\le& L_{LMBM}({\bf y,x,w}) + Kb + g(\bf x)\norm{\bf w}\nonumber\\
	L_{LMSBN}({\bf y,x,w}) &=& L_{LMBM}({\bf y,x,w}) = \sum_i [1-z_i]_+\label{eq:loss}\\
	z_i &=& \sum_{j:Y_i\in C_j}w_jf_j\nonumber\\
	b&^=&\log(e+e^{-1})\nonumber
\end{eqnarray}
\label{rem:approx}
\end{rem}
\begin{proof}
From Figure~\ref{fig:bounds}, it is easy to verify that $\log(1+e^{-z}) \le [1-z]_++b$, which leads to the first upper bound for SBN. For BM, because the features involves multiple variables ${Y_i}$ appear in all corresponding ${z_i}$, which makes the upper bounding much harder. Here we prove the second upper bound as follows:
\begin{eqnarray}
&&  L_{BM}({\bf y,x,w}) \nonumber\\
&=&-\frac{1}{2}\sum_iz_i+\log\sum_{\bf Y}e^{\frac{1}{2}\sum_iz_i}\nonumber\\
&=& -\frac{1}{2}\sum_iz_i+\log\sum_{{\bf Y}\setminus Y_1}e^{\frac{1}{2}\sum_{i\neq 1}z_i}\sum_{Y_1=\{-y_1,y_1\}}e^{\frac{1}{2}z_1}\nonumber\\
&\le&-\frac{1}{2}\sum_{i\neq 1}z_i + \log\sum_{{\bf Y} \setminus
Y_1}e^{\frac{1}{2}\sum_{i\neq 1}z_i}e^{[1-z_1]_++b}\nonumber\\
&\le&-\frac{1}{2}\sum_{i\neq 1}z_i + \log\sum_{{\bf Y} \setminus
\{Y_1,Y_2\}}e^{\frac{1}{2}\sum_{i\neq \{1,2\}}z_i}\nonumber\\
&&\quad \sum_{Y_2=\{-y_2,y_2\}}e^{\frac{1}{2}z_2}e^{[1-z_1]_++b}\nonumber\\
&\le&-\frac{1}{2}\sum_{i\neq \{1,2\}}z_i + \log\sum_{{\bf Y} \setminus
\{Y_1,Y_2\}}e^{\frac{1}{2}\sum_{i\neq \{1,2\}}z_i}\nonumber\\
&&\quad e^{[1-z_2]_++[1-z_1]_++2b+g_1({\bf x})\sum_{j:Y_1,Y_2\in C_j}w_j^2}\nonumber\\
&\le&\sum_i[1-z_i]_++Kb+g({\bf x})\sum_{j:C_j\in\mathcal{C}'}w_j^2
\end{eqnarray}
since the hinge loss for $Y_1$ also contains $Y_2$, when the partition function marginalizes $Y_2$, we have to relax the summation with a term proportional to the norm of the weights whose corresponding cliques include both $Y_1$ and $Y_2$. This relaxation is represented by $g_1({\bf x})\sum_{j:Y_1,Y_2\in C_j}w_j^2$, where $g_1$ is a constant determined by $\bf x$. After the whole partition function being relaxed, the upper bound contains a regularizer on all the weights whose corresponding cliques include at least two output $Y$. The set of all these cliques is $\mathcal{C}'$. \qed
\end{proof}

\begin{figure}%
\includegraphics[width=0.8\linewidth]{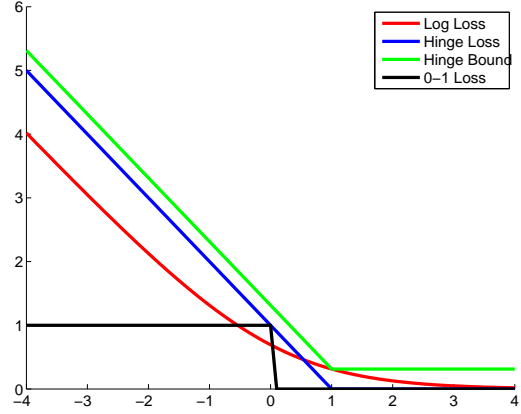}%
\caption{Losses and upper bound}%
\label{fig:bounds}%
\end{figure}

 Output values are predicted by minimizing the loss function, as shown in Equation~\ref{eqn:predict} below. With an
$\ell_2$ norm regularization on the weights $\bf w$, the training
problem for LMSBNs is defined as in Equation~\ref{eqn:learning}
below. Note that, for LMBMs, there is an extra $\ell_2$ regularization on the weights among the output $\bf Y$, but no regularization on the weights for individual $\bf Y$ or between $\bf X$ and $\bf Y$. 
%\begin{figure}
%\centering
%  \includegraphics[width=200pt]{bms}
% \caption{Examples of the conditional BMs and SBNs. A is a conditional BM; B is a SBN, with a topological order $\pi=(2,1,3)$}\label{fig:bms}
%\end{figure}
\begin{eqnarray}
\hat{\bf y} &=& \arg\min_{\bf y} L(\bf y,x,w)\label{eqn:predict}\\
\hat{\bf w} &=& \arg\min_{\bf w} \frac{1}{N}\sum_{l=1}^N L({\bf y}_l,{\bf x}_l,{\bf w})+R(\bf w)\label{eqn:learning}\\
\mbox{LMDBNs} &:& R(\bf w) \equiv \lambda \norm{\bf w}_2^2\nonumber\\
\mbox{LMBMs} &:& R(\bf w) \equiv \lambda \norm{\bf w}_2^2 + \lambda\eta_0 \norm{{\bf w}_{\mathcal{C}'}}_2^2\nonumber
\end{eqnarray}

\subsection{Generalization Bound}\label{sec:erm}
One major concern of structured prediction, as well as all classification problems, is generalization performance. Generalization performance for structured prediction has not been as well studied as for binary and multi-class classification \citep{Taskar03, Tsochantaridis04, daume09searn}. Both \citeauthor{Taskar03} and \citeauthor{Tsochantaridis04} employed the maximum-margin approach that builds on binary \emph{support vector machines} (SVMs). Generalization performance can be addressed by an upper bound on the prediction errors. However, the derivation of the bound is specifically restricted to the loss function they use, and hard to apply to other loss functions. \citeauthor{daume09searn} consider a sequential decision approach that solves the structured prediction problem by making decisions one at a time. These sequential decisions are made multiple times, and the output is obtained by averaging all results. The generalization bound is analyzed in terms of all these binary classification losses. One major drawback of this approach is that the averaged losses for the averaged classifiers need a large number of iterations to converge. Even if it converges, the bound is still loose compared to the bound we presented. We will discuss this further in Section~\ref{sec:relm}. 

In this section, we provide a general analysis tool for both single variable classification and structured prediction that allows arbitrary loss functions and holds tight. We first need the following threshold theorem:
\begin{thm}\label{thm:threshold}
Assuming ${\bf X,Y} \sim \mathcal{D}$,if \quad $\forall {\bf y}\neq \hat{\bf y}$, $\exists T>0$, s.t. $L({\bf x,y,w}) > T$, then
\begin{equation}
\Pr(\hat{\bf y}\neq {\bf y}|{\bf w}) \le \frac{1}{T}\E_{\mathcal{D}}[L({\bf x,y,w})]
\end{equation}
\end{thm}
\begin{proof}
\begin{eqnarray*}
\Pr(\hat{\bf y}\neq {\bf y}|{\bf w})
&=&\E_{\mathcal{D}}[\textbf{1}(\hat{\bf y}\neq {\bf y})]\\
&\le&\E_{\mathcal{D}}[\mathcal{H}(L({\bf x,y,w})-T)]\\
&\le&\frac{1}{T}\E_{\mathcal{D}}[L({\bf x,y,w})]
\end{eqnarray*}
The function $\textbf{1}(z)$ is the indicator function that is 1
when $z$ is true and 0 for false.  $\mathcal{H}(z)$ is the
Heaviside function that is 1 for $z\ge 0$ and 0 otherwise. The
last inequality comes from the fact that $\mathcal{H}(z-T)\le
\frac{z}{T}$.
\end{proof}

This threshold theorem allows one to discuss the prediction error bounds for any number of outputs with any loss function. For example, for
\emph{logistic regression} (LR), $L = \log(1+e^{-z}) > \log2$ whenever a mistake is made, so the threshold $T$ for LR is $\log2$. In \emph{Adaboost} \citep{zhang01statisticalbehavior}, $L=e^{-z}>1$ whenever it makes a mistake, so $T = 1$. For SSVMs, $L=[1-\max_{{\bf y'} \neq {\bf y}}(z({\bf y,x,w}) - z({\bf y',x,w}))]_+ > 1$ when it makes a mistake, so $T$ is again $1$. Then the prediction errors for all these classifiers are upper bounded by the expected loss divided by the threshold $T$. 

According to this theorem, the goal of all classification tasks is to find the hypothesis that predict with an expected loss as low as possible. 
On the other hand, for LMSBNs, there is a fast inference algorithm whose performance directly dependents on this quantity. The smaller the expected loss, the faster the inference. For both of the above reasons, the log-loss and exponential-loss are unfavorable because they are usually larger than zero even if the model fits the data well. Therefore, we choose the
hinge loss as the loss function for both LMSBNs and LMBMs.

The threshold for LMBMs is given in Remark~\ref{rem:lmbmthreshold}, and the threshold for LMSBNs is given in Remark~\ref{rem:lmsbnthreshold}. For a tight bound, the threshold should be large enough, so for LMBMs, we need to constrain the weights among the output variables. In other words, if the coupling between outputs is stronger than the coupling between an output and an input, then the possibility of overfitting increases. This also explains why the approximate loss of LMBMs contains regularizations for the coupling weights among the output variables. However, for LMSBNs, the threshold is always $1$. Generally speaking, LMSBNs can be expected to generalize better than LMBMs. 
\begin{rem}\label{rem:lmbmthreshold}
For LMBMs, $T = \min_i[\gamma-g({\bf x})\sum_{j:C_j\in\mathcal{C}'}w_j^2]_+ $, for some $g$.
\end{rem}
\begin{proof}
For any $y_i\neq \hat{y_i}$, we have $[1-z_i]_+=[1 - A_0-A_1]_+$,$[1-\hat{z}_i]_+=[1-A_0-A_2]_+$ where $A_0=\sum_{j:f_j=\hat{f}_j}w_jf_j$,\\ $A_1=\sum_{j:f_j\neq\hat{f}_j}w_jf_j$,\\$A_2=\sum_{j:f_j\neq\hat{f}_j}w_j\hat{f}_i$. Since all $\bf y$ takes $\{-1,+1\}$, so $A_2=-A_1$, and $[1-\hat{z}_i]_+=[1-A_0+A_1]_+$.\\
If $A_1<0$, we have $L>[1-z_i]_+>[1-A_0]_+$. Otherwise,
$L>\hat{L}>[1-\hat{z}_i]_+>[1-A_0]_+$. So
$L>[1-\sum_{j:f_j=\hat{f}_j}w_jf_j]_+$. We can further loosen it
to $L>\min_i[1 -g({\bf x})\sum_{j:C_j\in\mathcal{C}'}w_j^2]_+$.\qed
\end{proof}
\begin{rem}\label{rem:lmsbnthreshold}
For LMSBNs, $T = 1$
\end{rem}
\begin{proof}
Pick the first $Y_i$ in the topological order that does not equal the
optimal value, i.e. $y_i\neq \hat{y}_i$ and $\forall y_j\prec y_i,
y_j=\hat{y}_j$. Let $L_i = [1-z_i]_+$ and
$\hat{L_i}=[1-\hat{z}_i]_+$. Since $Y$ takes values $\{-1,1\}$
and only $y_i\neq\hat{y}_i$ in $z_i$, it is easy to verify that
$z_i=-\hat{z}_i$. So, we have $L_i=[1+\hat{z}_i]_+$.  If
$\hat{z}_i>0$, we have $L\ge L_i\ge 1$. Otherwise, $L\ge
\hat{L}\ge \hat{L_i}\ge 1$.\qed
\end{proof}

We assume all data are drawn from the same unknown distribution
$\mathcal{D}$. Since $\mathcal{D}$ is unknown, one can only minimize
the empirical risk rather than the expected risk. A fast convergence
rate of the empirical objective to the expected one was proved
in \citep{Shwartz08} for the single output variable case.  We can extend it to the general structured output case by
providing a structured \emph{Rademacher complexity} bound, as shown in Lemma ~\ref{lem:multiR}. 
\begin{lem}\label{lem:multiR}
Let $\mathcal{F}=\{{\bf x,y}\mapsto L(\bf x,y,w)\}$, \\$\mathcal{F}_i=\{{\bf x,y}\mapsto \sum_{j:Y_i\in C_j}w_jf_j\}$, $\phi(z) = [1 - z]_+$. We have
\begin{equation*}
\E \left[\sup_{h\in \mathcal{F}}\left(\E h - \hat{\E}_N h\right)\right]\le \sum_i \mathcal{R}_N(\phi\circ \mathcal{F}_i)
\end{equation*}
\end{lem}
\begin{proof}
\begin{eqnarray*}
&&\E \left[\sup_{h\in \mathcal{F}}\left(\E h - \hat{\E}_N h\right)\right]\\
&\le& \E \left[\sup_{h\in \mathcal{F}}\frac{1}{N}\sum_l \left(h({\bf x}_l',{\bf y}_l')-h({\bf x}_l,{\bf y}_l)\right)\right]\\
&\le&\E \left[\sum_i\sup_{h_i'\in \phi \circ \mathcal{F}_i}\frac{1}{N}\sum_l \left(h_i'({\bf x}_l',{\bf y}_l')-h_i'({\bf x}_l,{\bf y}_l)\right)\right]\\
&=&\sum_i\mathcal{R}_N(\phi\circ \mathcal{F}_i)
\end{eqnarray*}
Here $\mathcal{R}_N$ is the \emph{Rademacher
complexity}~\citep{Bartlett02} of sample size $N$.
See~\citep{Bartlett02} for details on the notation.\qed
\end{proof}
Together with Lemma~\ref{lem:multiR} and Corollary 4 in \citep{Shwartz08}, we can now derive a generalization bound as in
Theorem~\ref{thm:generalizationbound}.

\begin{thm}\label{thm:generalizationbound}
Let $\mathcal{L}({\bf w})=\E_{\mathcal{D}}[L(\bf x,y,w)]$,
${\bf w}_o=\arg\inf_{\bf w}\mathcal{L}({\bf w})$.
Assuming $\sum_jf_j^2<B^2$, for any $\delta>0$, with probability $1-\delta$ over the sample size $N$, if $\lambda=c\frac{B\sqrt{d/\delta}}{\norm{{\bf w}_o}_2\sqrt{N}}$, where $c$ is a constant, we have
\begin{eqnarray*}
\Pr(\hat{\bf y}\neq {\bf y}|\hat{\bf w})&\le& \frac{1}{T}\mathcal{L}(\hat{\bf w})\\
\mathcal{L}(\hat{\bf w})&\le&\mathcal{L}({\bf w}_o)+O\left(B\norm{{\bf w}_o}_2\sqrt{\frac{\log{(d/\delta)}}{N}}\right)
\end{eqnarray*}
\end{thm}

The basic idea of the structured \emph{Rademacher complexity} is to bound the whole functional space by a combination of the \emph{Rademacher complexity} of each subspaces. For LMBMs, a $f_j$ will be shared by all $z_i$ where $Y_i\in C_j$. So the subspaces overlap with each other, and the overall \emph{Rademacher complexity} counts the features multiple times while $B$ counts only once. Therefore the generalization bound is loosened by $\sqrt{d}$, where $d$ is the maximum clique size. The more complicated the graph, the larger the $d$. For LMSBNs, each feature only appears in one subspace, so $d$ is always $1$. Hence the bound for LMSBNs is tighter than for LMBMs. 

Furthermore, the bound given above is better than the PAC-Bayes bound of
SSVMs and is not affected by the inference algorithm. For SSVMs, when there is no cheap exact inference algorithm
available, the PAC-Bayes bound becomes worse due to the extra
degrees of freedom introduced by relaxations \citep{Kulesza2007},
leading to potentially poorer generalization performance. 

\subsection{Learning Algorithm}\label{sec:training}
For LMSBNs, the learning problem defined in
Equation~\ref{eqn:learning} can be decomposed into $K$ independent
optimization problems\footnote{$\lambda$ should be the same, otherwise
Theorem~\ref{thm:generalizationbound} does not hold.}. Each of them
can be solved efficiently by any of the modern fast solvers such as the dual
coordinate descent algorithm \citep{Hsieh2008} (DCD), the primal
stochastic gradient descent algorithm
(PEGASOS) \citep{Shai07,bottou-bousquet-2008} or the exponentiated
gradient descent algorithm \citep{Collins-08}. For LMBMs, the weights are shared in multiple $z_i$, one has to optimize the whole objective simultaneously. Similar to \citep{Hsieh2008}, we give a \emph{dual coordinate descent} based optimization algorithms for LMBMs.

Consider the following primal optimization problem:
\begin{eqnarray*}
  \min_{\bf w,\xi} &&\frac{1}{2}\sum_j\eta_j w_j^2+\frac{1}{\lambda N}\sum_{il}\xi_{il} \\
 \mbox{subject to}&& \sum_{j:Y_i\in C_j}w_jf_{jl}\ge 1 - \xi_{il}\\
   && \xi_{il}\ge 0
\end{eqnarray*}
where $\eta_j=1$ if $w_j$ is not extra regularized; otherwise, $\eta_j=1+\eta_0$. The index $l$ represents each training data. Let $\alpha_{il}$ and $\beta_{il}$ be Lagrange multipliers. Then, we have the
Lagrangian:
\begin{eqnarray*}
    \mathbb{L}({\bf w}, \xi,\alpha,\beta) &=&\frac{1}{2}\sum_j\eta_j
    w_j^2+\frac{1}{\lambda N}\sum_{il}\xi_{il}-\sum_{il}\beta_{il}\xi_{il}-\\
&&\sum_{il}\alpha_{il}(\sum_{j:Y_i\in C_j}w_jf_{jl} - 1 +
    \xi_{il})
\end{eqnarray*}
We optimize $\mathbb{L}$ with respect to $\bf w$ and $\xi$:
\begin{eqnarray*}
  \frac{\partial \mathbb{L}}{\partial w_j} &=& \eta_jw_j - \sum_{l}\sum_{i:Y_i\in C_j}\alpha_{il}f_{jl}=0\\
  \frac{\partial \mathbb{L}}{\partial \xi_{il}} &=&
  \frac{1}{\lambda N}-\alpha_{il}-\beta_{il}=0
\end{eqnarray*}
Substituting for $\bf w$ and $\xi$, we have the dual objective:
\begin{eqnarray*}
 \mathbb{L}_\alpha
&=&\frac{1}{2}\sum_{j,l,l'}\sum_{i:Y_i\in C_j}\sum_{i':Y_{i'}\in C_j}\alpha_{il}\alpha_{i'l'}Q_{jll'}-\sum_{il}\alpha_{il}
\end{eqnarray*}
where $Q_{jll'} = \frac{f_{jl}f_{jl'}}{\eta_j^2}$.
The dual coordinate descent algorithm picks $\alpha_{il}$ one at a
time and optimizes the dual Lagrangian with respect to this variable. 
The resulting algorithm is described in Algorithm 1.
\begin{algorithm}
\label{alg:dcd} \caption{The dual coordinate descent algorithm for
large margin Boltzmann machines}
\begin{algorithmic}[1]
\REQUIRE $\set{f_{jl}},\set{Q_{jll}},\lambda, N$
\ENSURE $\bf w$
\STATE $\alpha\leftarrow 0,{\bf w}\leftarrow 0$\\
\WHILE{$\alpha$ is not optimal}
\FORALL{$\alpha_{il}$} \STATE
$\alpha_{o} \leftarrow \alpha_{il}$\\
\STATE$G =
\sum_{j:Y_i\in C_j}w_jf_{jl}-1$\\
\STATE$PG = \left\{\begin{array}{cc}
\min(G,0) & \alpha_{o} = 0, \\
 \max(G,0) & \alpha_{o} = \frac{1}{\lambda N}, \\
 G & 0 < \alpha_{o}< \frac{1}{\lambda N}\\
\end{array}\right.$\\
\IF{$|PG|\neq 0$} \STATE $\alpha_{il} \leftarrow
\min(\max(\alpha_{o} - \frac{G}{\sum_{j:Y_i\in C_j}Q_{jll}},0),\frac{1}{\lambda N})$\\
\STATE $w_j \leftarrow w_j + (\alpha_{il} -
\alpha_{o})f_{jl}, \forall j:Y_i\in C_j$\\
\ENDIF \ENDFOR\ENDWHILE \STATE {\bf return} $\bf w$
\end{algorithmic}
\end{algorithm}

\subsection{Inference Algorithm}\label{sec:bb}

In this section, we propose a simple and efficient inference algorithm
 (Algorithm 2) to solve the prediction problem in
 Equation~\ref{eqn:predict} for LMSBNs.  According to the topological order
 of the graph, we branch on each $Y_i$, and compute $z_i$ with $\bf x$
 and all of its parents $y_j$. We first try the value of $y_i$ that
 makes $z_i>0$, i.e., the left branch in the algorithm, then the right
 branch with the opposite value of $y_i$. During this search, we keep
 an upper bound initialized to a parameter $S\ge 1$. Whenever the current
 objective is higher than the upper bound, we backtrack to the
 previous variable. The search terminates before $K^S$ states of $\bf
 Y$ have been visited. 
\begin{algorithm}
\caption{The Branch and Bound Algorithm for
Inference}\label{alg:bb} 
\begin{algorithmic}[1]
\REQUIRE $\bf x,w$, $S\ge 1$
\ENSURE $\hat{\bf y},UB$
\STATE $UB = S$,\,$i=0$,\,$\bf U = 0$;
\WHILE{$i>=0$}
   \IF{$i=K$} \IF{$U_K<UB$} \STATE $UB = U_K$,\, $\hat{\bf y} = \bf y$;\ENDIF \STATE $i=i-1$;
   \ELSE
      \IF {Left branch has not been tried}
         \STATE $y_i= \arg\max_{y_i}z_i$,\, $U_{i+1}= U_i + [1-z_i]_+$;
      \ELSIF {Right branch has not been tried}
         \STATE $y_i= -y_i$,\,$U_{i+1}=U_i+[1+z_i]_+$;
      \ENDIF
      \IF{$U_{i+1}\ge UB$ or both branches have been tried}
         \STATE $i=i-1$;
      \ELSE
         \STATE $i=i+1$;
      \ENDIF
   \ENDIF
\ENDWHILE
\end{algorithmic}
\end{algorithm}
The following theorem shows that with a high probability, the above
 algorithm computes the optimal values in polynomial time:
\begin{thm}
\label{thm:timecomplexity}
For any $S\ge 1$, the BB algorithm reaches the optimal
values before $O(K^S)$ states are visited with a probability at least
$1-\frac{1}{S}\mathcal{L}({\bf w})$.
\end{thm}
\begin{proof}
During the search, if we branch on the right, the hinge loss
$[1+z_i]_+$ is greater than 1. So, for a given $\bf x$, if the true
objective $L<S$, the optimal objective $\hat{L}<L<S$ as well, and the
optimal path contains at most $S$ right branches. Since the BB
algorithm always searches the left branch first, the optimal path will
be reached before $\sum_{i=0}^{S-1}\left(
\begin{array}{c}
i\\
K
\end{array}\right)\le O(K^S)
$ states have been searched. According to the \emph{Markov inequality}, $\Pr(L<S)= 1-\Pr(L>S)\ge 1-\frac{1}{S}\mathcal{L}(\bf w)$.
\end{proof}

The BB algorithm adjusts the search tree according to the model
weights. Through training, optimal paths are condensed to the low
energy side, i.e., the left side of the search tree with a high
probability. This probability is directly related to the expected
loss with respect to the given data distribution. We therefore label
the BB algorithm a \emph{data-dependent} inference algorithm. Most
popular inference algorithms for exact or approximate inference depend
on graph complexity: the more complicated the graph, the slower the
inference. This trade-off diminishes the applicability of these
algorithms and presents researchers with the difficult problem of
selecting a (possibly sub-optimal) graph structure that balances the
accuracy and the efficiency. The BB algorithm for LMSBNs circumvents
this trade-off and allows arbitrary complicated graphs without
sacrificing computational efficiency. In fact, if a particular
complicated graph yields a smaller expected loss, the BB algorithm in
turn runs even faster.

It is well-known that for NP-hard problems, there may be many instances
that can be solved efficiently. The area of \emph{speedup learning}
focuses on learning good heuristics to speedup problem solvers. The
approach presented here can be regarded as a novel method for
\emph{speedup learning} \citep{Tadepalli96a} and demonstrates that the
experience gained during training can speedup a problem solver
significantly.

The BB algorithm is specifically designed for LMSBNs, a directed graphical 
model. For undirected models, the BB algorithm does not guarantee a
polynomial time complexity with a high probability. Indeed, we observe
an exponential time complexity when it is applied to LMBMs. For the undirected models including SSVMs and LMBMs, we implement a \emph{convex relaxation-based linear programming} (LP) \citep{Wainwright05a}. Note that although LMBMs don't have a fast inference algorithm, unlike SSVMs, the learning is not affected by the inference algorithm. In the experiments section, we will show that LMBMs outperforms SSVMs. 

The BB algorithm differs from other search-based decoding
algorithms, e.g., beam search and best first search \citep{Abdou2004}, in several
aspects. First, those search algorithms typically prune the supports
of maximum cliques that can grow exponentially. On one hand, the pruning
can lead to misclassification quickly if backtracking is not implemented.
On the other hand, the number of remaining states might still be large
so that the inference is still slow.  Furthermore, even if a backtracking
procedure is implemented, unlike the BB algorithm for LMSBNs, there
are still no guaranteed heuristics that can prune the states efficiently
and correctly.

To demonstrate the efficiency and the data dependency property, we run
the algorithms on the test data of \textbf{RCV1-V2} (a text categorization dataset) with a trained
model and a random untrained model. The running times are collected by varying
the number of output variables. The CPU time is measured on a 2.8Ghz
Pentium4 desktop computer.
\begin{figure}[!t]
\centering \includegraphics[width=0.8\linewidth]{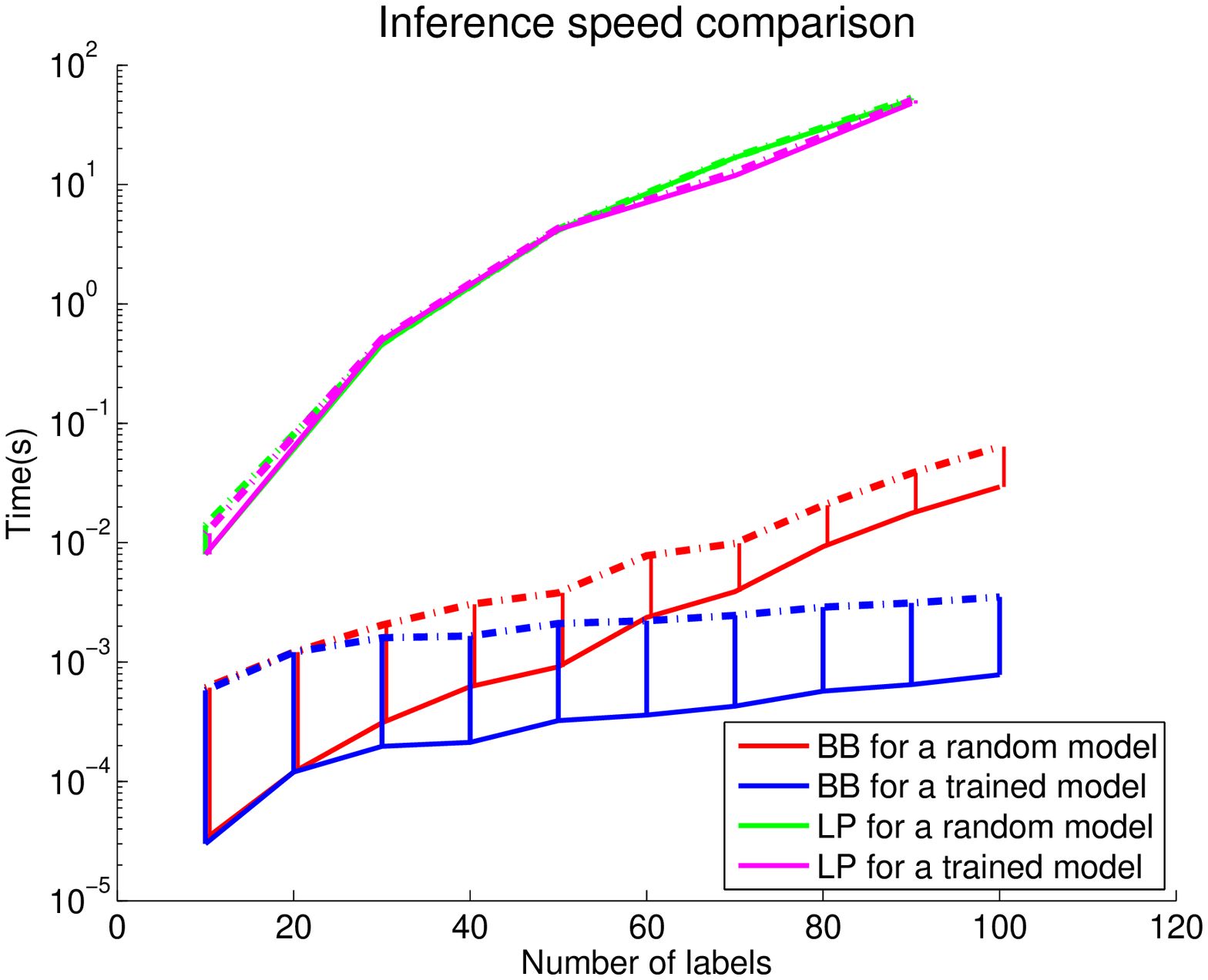}
  \includegraphics[width=0.8\linewidth]{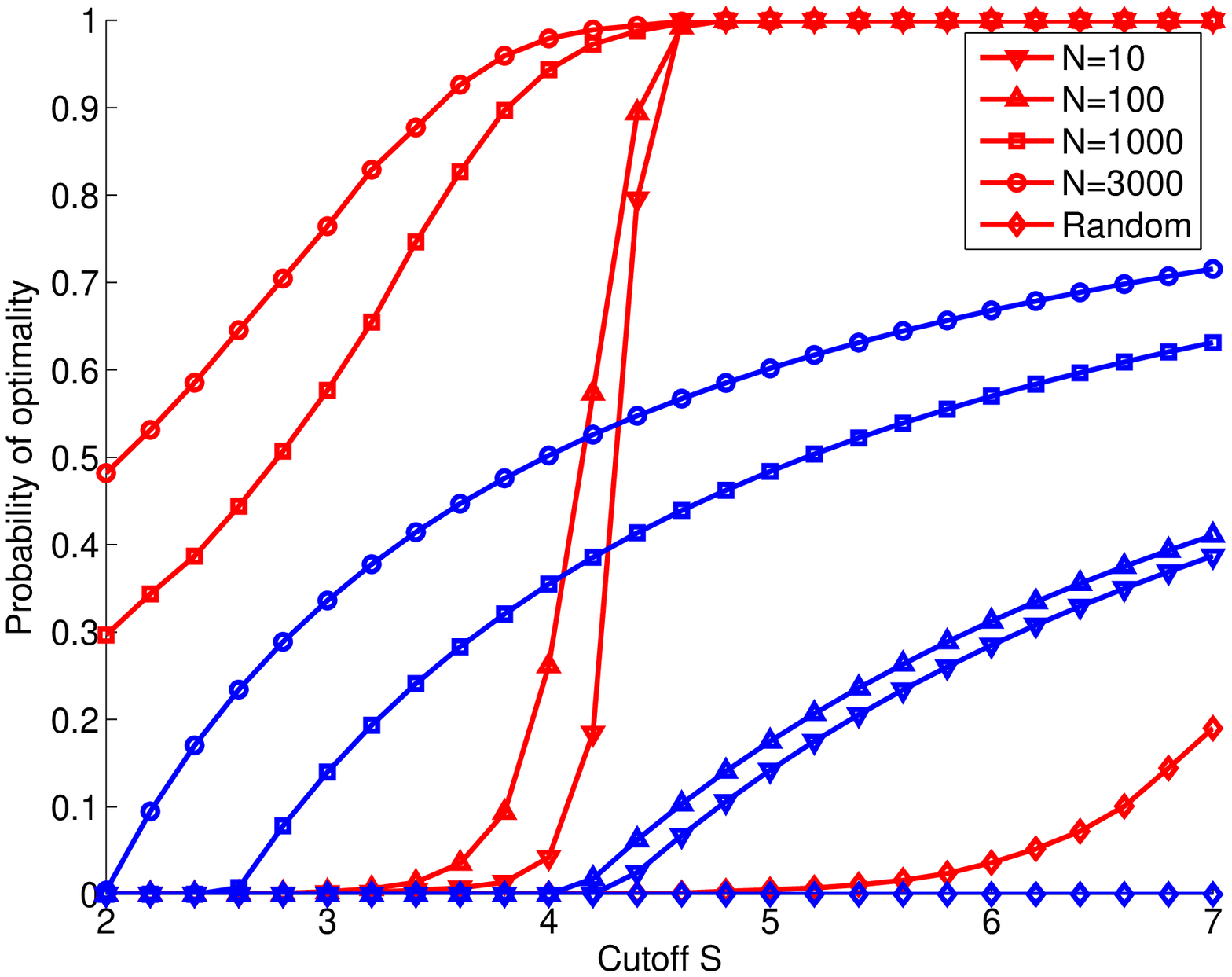}
  \caption{Upper graph. Running time comparisons of LP and BB
  algorithms on the test dataset of \textbf{RCV1-V2}. The dashed lines
  are 1 standard deviation above the mean. The time axis is
  log-scaled. Lower graph. Accuracy and data distribution dependency. The more the training data, the more accurate and faster the prediction. The
  corresponding estimated theoretical lower bounds are plotted in blue.}\label{fig:inf_spd}
\end{figure}

The upper graph in Figure~\ref{fig:inf_spd} demonstrates that the BB
algorithm performs several orders of magnitude faster than
LP\footnote{The speed measurement of LP is comparable to Finley et
al. \citep{Finley/Joachims/07a}. According to their experiments,
\emph{graph cuts} and \emph{loopy belief propagation} can perform
10-100 times faster, but are still much slower than BB.}. In this
experiment, $S$ is set to a very large value such that the solution
from BB is guaranteed to be the optimal solution. The running time of
LP with respect to the number of output variables does not vary from a
trained model to a random untrained model, but the running time of BB
changes significantly. For the random untrained model, the BB algorithm
demonstrates an exponential time complexity with respect to the number
of output variables. However, after training, the running time of the BB
algorithm scales up much more slowly.

This observation underscores the data distribution dependent property
of BB, i.e., the better the model fits the data, the faster BB
performs. We illustrate this property further by a second
experiment. In this experiment, the probability of the BB algorithm
reaching the optimal values is plotted by varying the cutoff threshold
$S$. According to Theorem~\ref{thm:timecomplexity}, $S$ reflects the
running time overhead for the BB algorithm. We compare this curve for
several models, namely, a random model and models trained with 10,
100, 1000, and 3000 training instances respectively. The lower graph
in Figure~\ref{fig:inf_spd} shows a significant improvement for the
trained models over the random untrained model. Moreover, with more
and more training instances, more and more test instances can be
predicted exactly and quickly. In the same figure, we also plot the
corresponding theoretical lower bounds estimated from the testing
dataset (blue lines). The lower graph of Figure~\ref{fig:inf_spd}
verifies Theorem~\ref{thm:timecomplexity} empirically.

Due to the fast and accurate inference algorithm for LMSBNs, we can start with
the most complicated graph structure, i.e., a fully connected
model. The linear form of $z_i$ can be generalized to high order features. Moreover, the kernel trick can be applied to augment the modeling power. The only thing one needs to concern is to minimize the expected loss as much as possible because the small expected loss guarantees not only a high prediction accuracy but also a fast inference.

\section{Related Models}\label{sec:relm} 
Most \emph{maximum margin estimated} structured prediction models, e.g.,
SSVMs \citep{Tsochantaridis04}, \emph{maximum margin Markov networks}
(M3Ns) \citep{Taskar03}, \emph{Maximum margin Bayesian networks} (M2BNs) \citep{Guo2005} and \emph{conditional graphical models}
(CGMs) \citep{zoubin07} adopt a min-max formulation as shown below:
\begin{eqnarray}
\min_{\bf w}&&\frac{1}{2}\norm{\bf w}^2+\frac{1}{\lambda N}\sum_l[\Delta({\bf y}_l, {\bf y}) - m({\bf x}_l,{\bf y}_l;\bf w)]_+\label{eqn:ssvm}\\
m({\bf x}_l,{\bf y}_l;{\bf w})&=& \max_{{\bf y}\neq {\bf y}_l}\Psi({\bf x}_l,{\bf y}_l;{\bf w})-\Psi({\bf x}_l,{\bf y};{\bf w})\nonumber
\end{eqnarray}
where $\Psi$ is the compatibility function derived from a probabilistic model, and $m$ is the margin function.

The embedded \emph{maximization} operation potentially induces an
exponential number of constraints. This exponential number of
constraints makes optimization intractable. In M2BNs, the local normalization constraints makes the problem even harder. SSVMs utilize a cutting
plane algorithm \citep{Joachims08a} to select only a small set of
constraints. M3Ns directly treat the dual variables as the
decomposable pseudo-marginals. When the undirected graph is of low
tree-width, both SSVMs and M3Ns are computationally efficient and
generalize well. However, for high tree-width, approximate inference
has to be used and both the computational complexity and the sample
complexity increase significantly \citep{Kulesza2007,Finley/Joachims/07a}.

CGMs decompose the single hinge loss into a summation of several hinge
losses, each corresponding to one feature function, such that the
exponential number of combinations is greatly reduced. The decomposition from one hinge loss to multiple hinge loss is similar to LMBMs and LMSBNs. However, CGMs decompose to each feature function. For real problems, not every feature function could be compatible to the data, which leads to a large and trivial upper bound. Therefore, the performance can not be guaranteed. 

The \emph{large margin estimation} by the threshold theorem ~\ref{thm:threshold} generalizes the \emph{maximum margin estimation} approach. As long as the loss function satisfies the threshold theorem, there is a margin function implicitly defined such that minimizing the expected loss maximizes the margins. The traditional log-loss based models, e.g., CRFs and MEMMs, can be discussed under the \emph{large margin estimation} framework, but the thresholds are possibly small so that the upper bounds become trivial. This suggests that \emph{large margin estimated} models could generalize better than \emph{maximum likelihood estimated} models.

For problems like semantic annotation, a low-treewidth graph usually is insufficient to
represent the knowledge about the relationships among the labels. The
example in Figure~\ref{fig:annotation} illustrates the motivation for
a high-treewidth graph.  All of the models discussed above lack a fast and accurate inference algorithm for high-treewidth graphs, and are subject to the trade-off between the treewidth and computational efficiency.

To speed up inference for a high-treewidth graphical models, one can use mixture models to represent probabilities. For example, MoP-MEMMs \citep{Rosenberg07} extend MEMMs to address long-range dependencies
and represent the conditional probability by a mixture model. \citeauthor{Wainwright03treereweighted} uses a mixture of trees to approximate a \emph{Markov random fields}. Both demonstrate performance gains but one still has to improve inference speed by restricting the number of mixtures.

Another line of research for high-treewidth graphical models uses \emph{arithmetic circuits} (AC) \citep{Darwiche00adifferential} to represent the Bayesian networks. The AC inference is linear in the circuit size. As long as the circuit size is low, the inference is fast. But learning the optimal AC is an NP-hard problem. Similarly, one has to improve inference speed by penalizing the circuit size \citep{Lowd08}.

The \emph{search based structured prediction} (SEARN) \citep{daume09searn} takes a different approach than \emph{probabilistic graphical models} to handle the high tree-width graphs. It solves the structured prediction by making decisions sequentially. The later classifier can take all the earlier decisions as inputs, which is similar to LMSBNs. In fact, the inference can be considered as the initial decision of the BB algorithm. The expected errors caused by this naive inference could be very high. SEARN implements an averaging approach to reduce the expected errors. It trains a set of sequential classifiers for each iteration and outputs the prediction by averaging the decisions made over all iterations. The earlier decisions will be fed into later classifiers, so the later classifiers possibly make fewer mistakes. By averaging over iterations, the expected loss are reduced thereafter. Roughly speaking, the prediction errors will be bounded by this averaged expected loss\footnote{The expectation is over the unknown data distribution, while the averaging is over the iterations.} multiplied by $\log K$\footnote{Suppose that the initial policy can make perfect predictions.}. Compared to the bounds of LMBMs and LMSBNs, where the prediction errors are bounded by the minimum expected loss divided by the threshold $T$, the generalization bound of SEARN is rather loose. Furthermore, according to \citep{daume09searn}, one needs a large number of iterations to reach that bound which slows down the inference. Therefore, one still has to limit the number of iterations for a faster inference, which might sacrifice the prediction accuracy.

Unlike all the above approaches, LMSBNs possess a very interesting property  that one does not have any constraints on the modeling power. The smaller the expected loss, the faster the inference. Usually, one obtains a smaller expected loss by using a more complicated graph. This property leads to a novel approach for structured prediction with high tree-width graphs.

\section{Experiments}\label{sec:experiments}
The performance of LMSBNs was tested on a scene annotation problem
based on the \textbf{Scene} dataset \citep{Boutell04}. The dataset contains
1211 training instances and 1196 test instances. Each image is
represented by a 294 dimensional color profile feature vector (based
on a CIE LUV-like color space). The output can be any combination of 6
possible scene classes (beach, sunset, fall foliage, field,
urban, and mountain).

We compare a fully connected LMSBN with three other methods:
\emph{binary classifiers} (BCs), SSVMs \citep{Finley/Joachims/07a},
\emph{threshold selected binary classifiers} (TSBCs) \citep{Fan07}. 
BCs train one classifier for each label and
predict independently. For
SSVMs, we follow \citep{Finley/Joachims/07a} to implement a fully
connected undirected model with binary features. We implement a
\emph{convex relaxation}-based linear programming algorithm for
inference, since in both  \citep{Finley/Joachims/07a}
and \citep{Kulesza2007}, the \emph{convex relaxation}-based approximate
inference algorithm was shown to outperform other approximate
inference algorithms such as \emph{loopy belief propagation} and
\emph{graph cuts} \citep{Kolmogorov02}. TSBCs iteratively tune the optimal decision
threshold for each classifier to increase the overall performance with
respect to a certain measure, e.g., \emph{exact match ratio} and
\emph{F-scores}. Many labels in the multi-label datasets are highly
unbalanced, leading to classifiers that are biased. TSBCs can
effectively adjust the classifier's precision and recall to achieve
state-of-the-art performance. In our comparisons, we borrow the best
results from \citep{Fan07} directly.

We implemented two BCs, a linear BC
(BCl) and a kernelized BC (BCk), and three LMSBNs: (1) LMSBNlo is
trained with default order, i.e., ascending along the label indices;
(2) LMSBNlf is trained with the order selected according to the
\emph{F-scores} of the BC. We sort the variables according to their
\emph{F-scores} of the BC. The higher the \emph{F-score}, the smaller
the index in the order; (3) LMSBNkf is a kernelized model with the
same order as LMSBNlf. We also implemented two SSVMs: (1) SSVMhmm is
trained by using a first-order Markov chain. It is different from the 
$SSVM^{hmm}$ package that does not consider all inputs $\bf X$
for each $Y_i$. The inference algorithm for SSVMhmm is the Viterbi
algorithm; (2) SSVMfull is trained by using a fully connected graph.

We consider three categories of
performance measures. The first consists of instance-based measures and includes  the \emph{exact match ratio} (E) (Equation \ref{eqn:exact}) and instance-based
\emph{F-score} (Fsam) (Equation \ref{eqn:f}). The second consists of label-based measures and includes the \emph{Hamming loss} (H) (Equation \ref{eqn:hamm}) and the \emph{macro-F score}
(Fmac) (Equation \ref{eqn:fma}). The last is a mixed measure, the \emph{micro-F score}
(Fmic) (Equation \ref{eqn:fmi}). Fsam calculates the \emph{F-score} for each instance, and
averages over all instances. Fmac calculates the \emph{F-score} for
each label, and averages over all labels. Fmic calculates the
\emph{F-score} for the entire dataset. 

\begin{eqnarray}
E&=&\frac{1}{N}\sum_l\textbf{1}({\bf y}_l= {\bf \hat{y}}_l)\label{eqn:exact}\\
H&=&\frac{1}{NK}\sum_{li}\textbf{1}(y_{li}\neq \hat{y}_{li})\label{eqn:hamm}\\
Fsam&=&\frac{1}{N}\sum_{l}\frac{2\sum_i\textbf{1}(y_{li}=\hat{y}_{li}=1)}{\sum_i(\textbf{1}(y_{li}=1)+\textbf{1}(\hat{y}_{li}=1))}\label{eqn:f}\\
Fmac &=&\frac{1}{K}\sum_{i}\frac{2\sum_l\textbf{1}(y_{li}=\hat{y}_{li}=1)}{\sum_l(\textbf{1}(y_{li}=1)+\textbf{1}(\hat{y}_{li}=1))}\label{eqn:fma}\\
Fmic &=&\frac{2\sum_{il}\textbf{1}(y_{li}=\hat{y}_{li}=1)}{\sum_{il}(\textbf{1}(y_{li}=1)+\textbf{1}(\hat{y}_{li}=1))}\label{eqn:fmi}\\
\end{eqnarray}

The instance-based measure is
more informative if the correct prediction of co-occurrences of labels
is important; the label-based measure is more informative if the
correct prediction of each label is deemed important.

\begin{figure*}%
\centering
\includegraphics[width=0.45\linewidth]{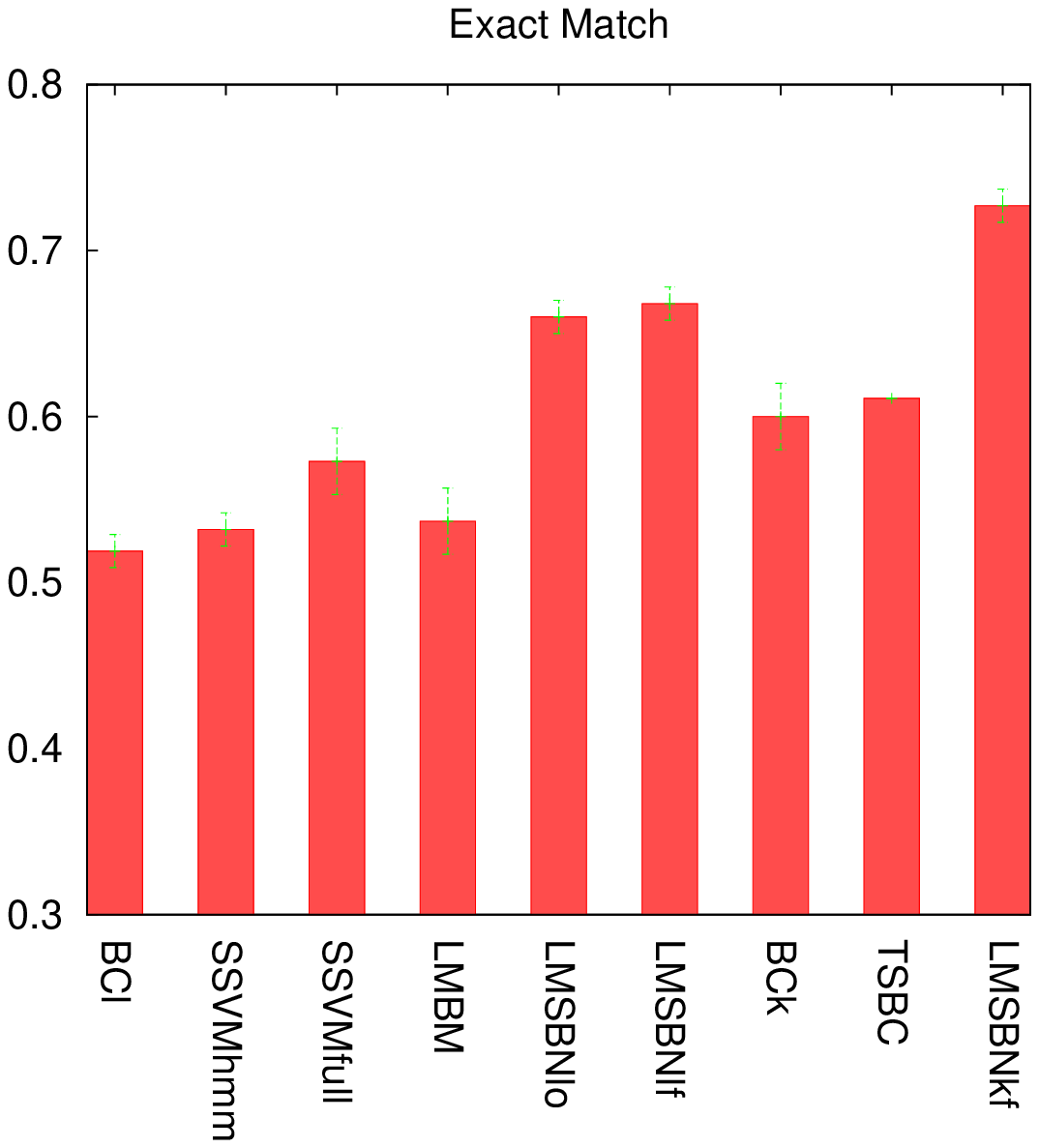}
\includegraphics[width=0.45\linewidth]{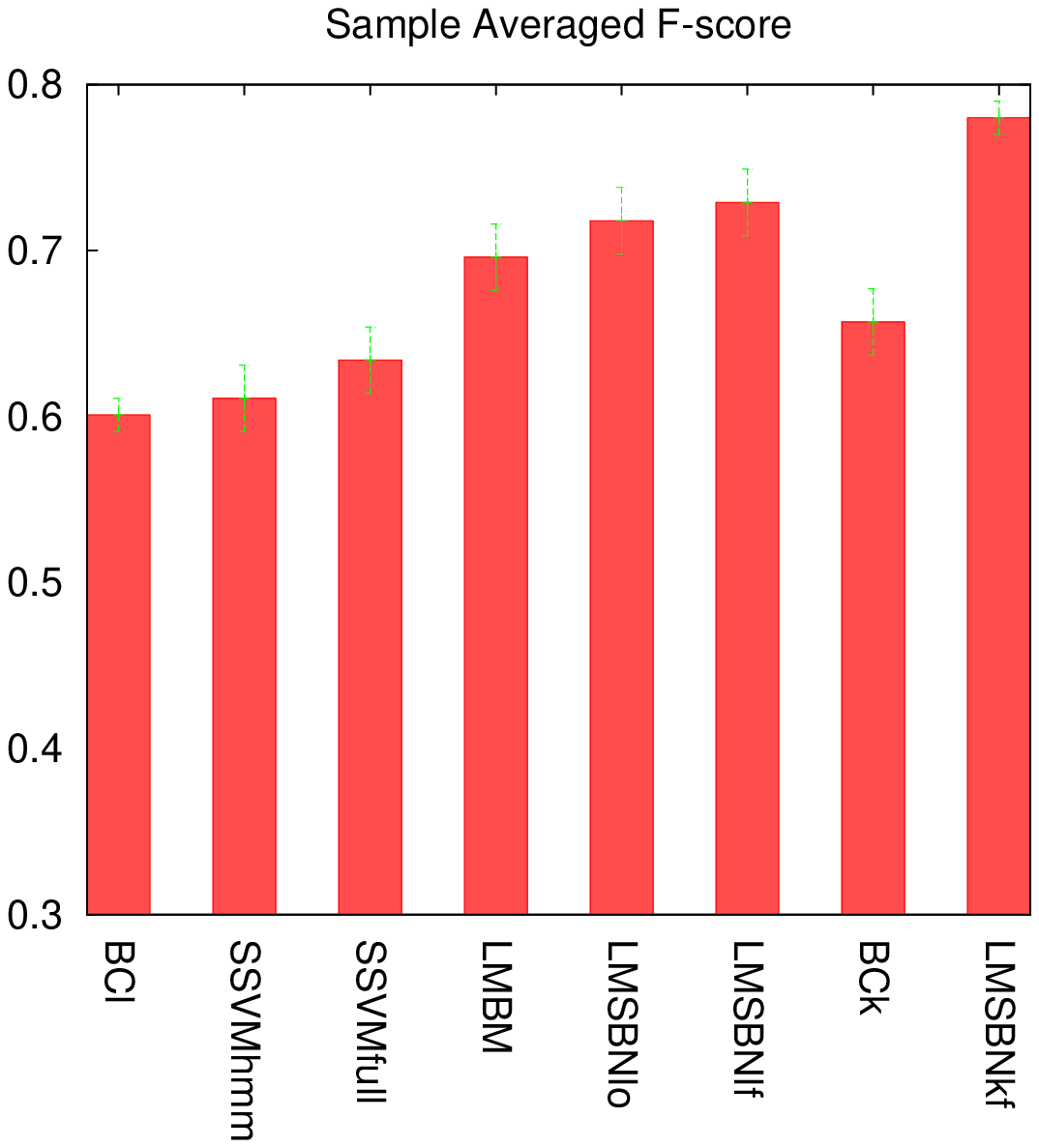}\\
\includegraphics[width=0.45\linewidth]{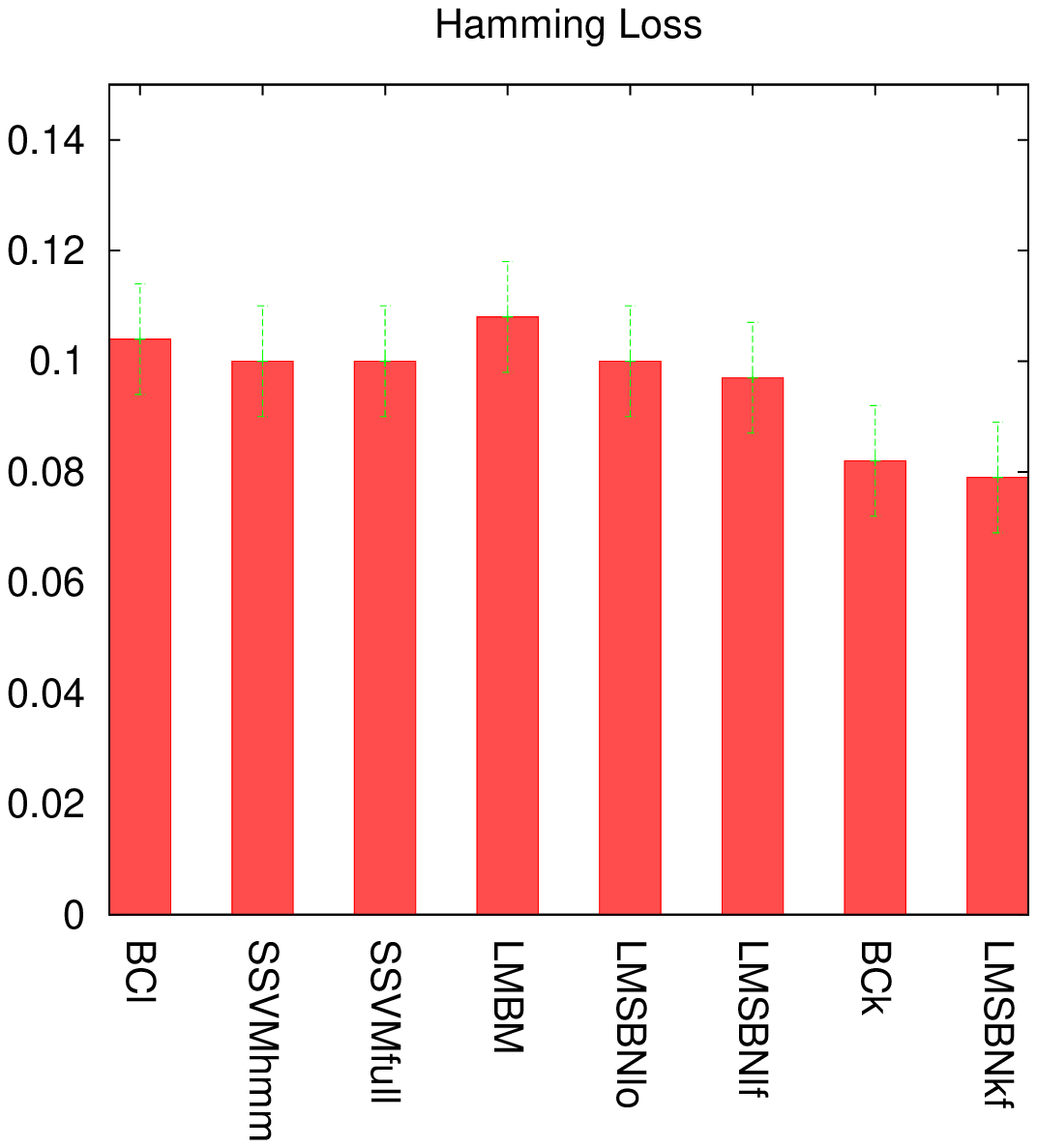}
\includegraphics[width=0.45\linewidth]{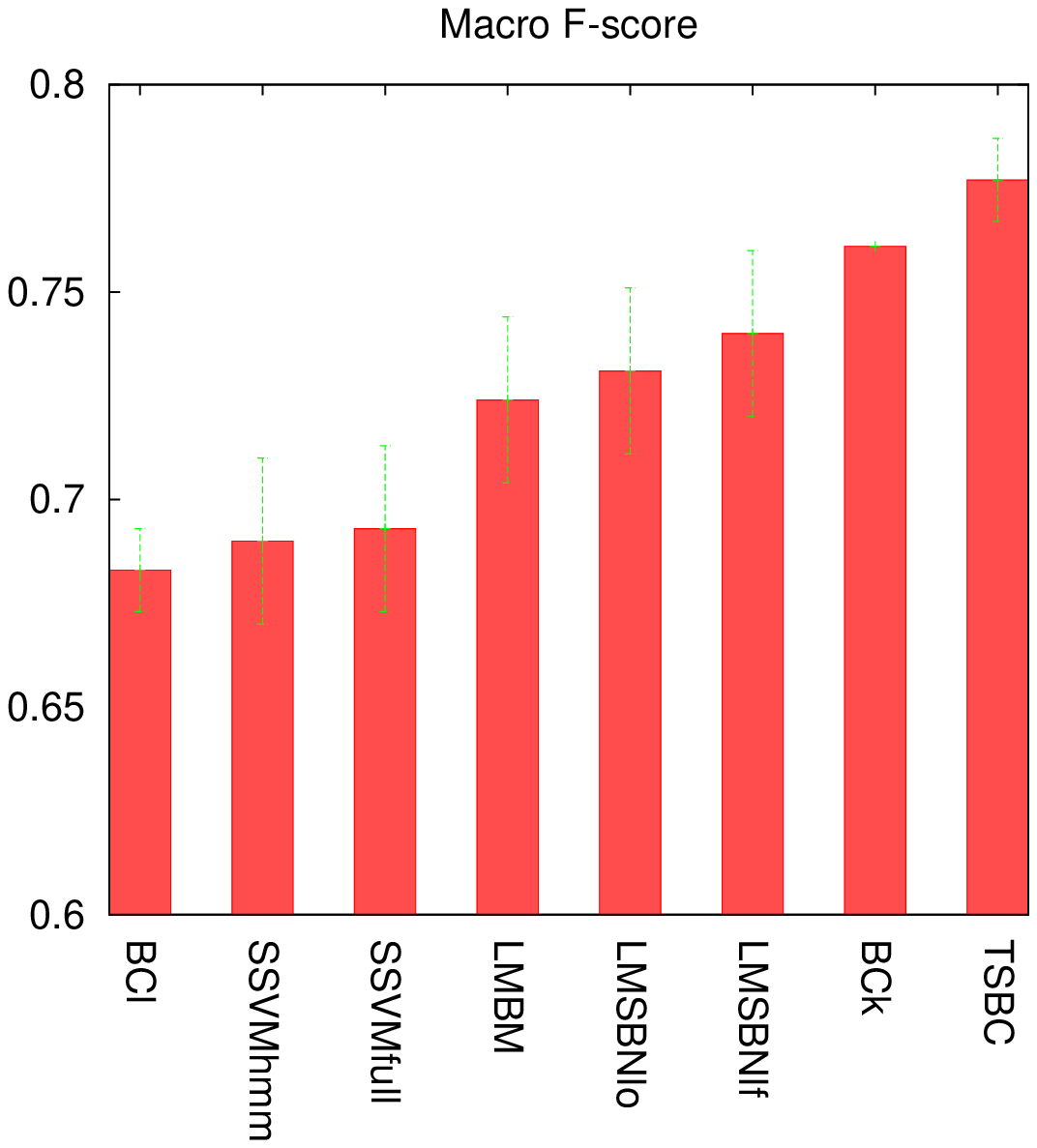}\\
\includegraphics[width=0.45\linewidth]{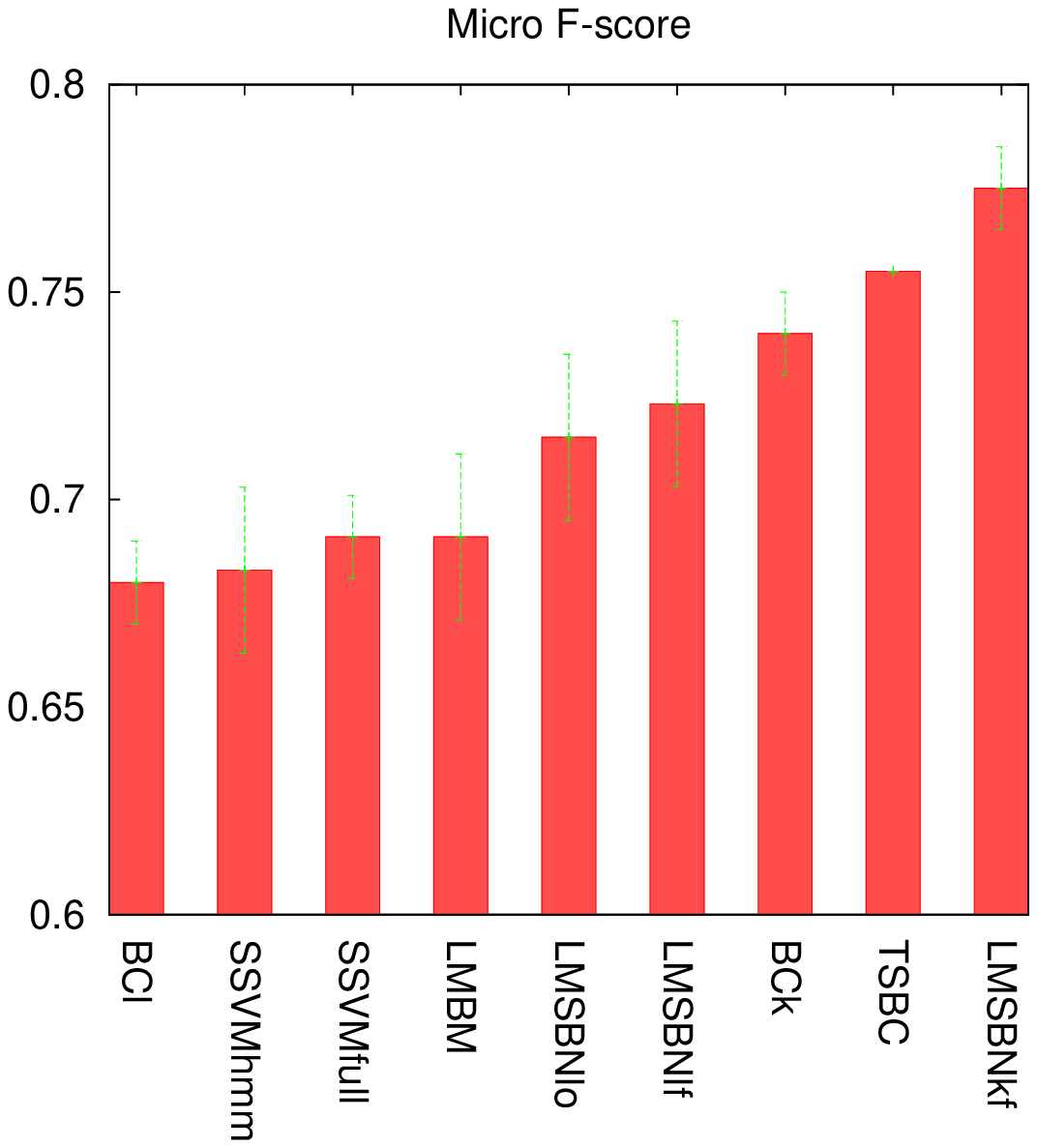}
\caption{Performance Comparisons. Note: For Hamming loss, the smaller the value, the better the performance.}%
\label{fig:perfs}%
\end{figure*}

The results are shown in Figure~\ref{fig:perfs}. LMSBNkf consistently
performs the best on all measures. Even the LMSBN models
without kernels outperform TSBC on instance-based measures. 

SSVMhmm performs better than the BCl, but worse than the SSVMfull as expected. The inference speed of BCl is
faster than SSVMhmm, which in turn is faster than SSVMfull. This
demonstrates the trade-off between modeling power and efficiency. 

With the help of kernels, LMSBNkf further outperforms the
TSBC on all measures. LMSBNs as proposed in this paper are geared
towards minimizing 0-1 errors. Threshold tuning is particularly effective in the case of
highly unbalanced labels. An interesting line of research is combining LMSBNs
with threshold tuning to further improve the performance. 

\section{Conclusions}
This paper proposes the use of large margin graphical models for fast
structured prediction in images with complicated graph structures. A major advantage of the proposed approach is the
existence of fast training and inference algorithms, which open the
door to tackling very large-scale image annotation problems. Unlike
previous inference algorithms for structured prediction, the proposed
BB inference algorithm does not sacrifice representational power for
speed, thereby allowing complicated graph structures to be
modeled. Such complicated graph structures are essential for accurate
semantic modeling and labeling of images. Our experimental results
demonstrate that the new approach outperforms current state-of-the-art
approaches. Future research will focus on applying the framework to
annotating parts of images with their spatial relationships, and
enhancing the representational power of the model by introducing
hidden variables.

\bibliographystyle{spbasic}
\bibliography{lmbm}

\end{document}